\documentclass[11pt]{article}

\usepackage[numbers, compress]{natbib}

\usepackage{fullpage,times}
\usepackage{parskip}

\usepackage{lipsum}  
\makeatletter
\newcommand{\printfnsymbol}[1]{%
  \textsuperscript{\@fnsymbol{#1}}%
}
\makeatother

\title{\huge {\Name} Flow: 
A Marginal Preserving Approach to Optimal Transport 
}

\author{
Qiang Liu \\ 
    ~University of Texas at Austin\\
	\texttt{lqiang@cs.utexas.edu}
}

\date{}

\usepackage{qiangstyle}

\definecolor{commentcolor}{RGB}{110,154,155}   
\definecolor{inputcolor}{RGB}{255, 105, 180}   

\begin{document}

\maketitle

\begin{abstract}
We present a flow-based approach to 
the optimal transport (OT) problem 
between two continuous distributions $\tg_0,\tg_1$ on $\RR^d$,
of minimizing a transport cost $\E[c(X_1-X_0)]$ 
in the set of couplings $(X_0,X_1)$ whose marginal distributions on $X_0,X_1$ equals $\tg_0,\tg_1$, respectively, 
where $c$ is a cost function.
Our method iteratively constructs a sequence of neural ordinary differentiable equations (ODE),  
each learned by solving a simple unconstrained regression problem, which 
monotonically reduce the transport cost while automatically preserving the marginal constraints. 
This yields a monotonic 
\emph{interior} approach that traverses inside the set of valid couplings to decrease the transport cost, 
which distinguishes itself from most existing approaches that enforce the coupling constraints from the outside. 
The main idea of the method draws from 
\emph{rectified flow} \cite{rectified}, 
a recent approach 
that simultaneously decreases 
the whole family of  
transport costs induced by convex functions $c$
(and is hence \emph{multi-objective} in nature), 
but is not tailored to minimize a  specific transport cost. 
Our method is a \emph{single-object} variant of rectified flow 
that guarantees to solve the OT problem 
for a fixed, user-specified convex cost function $c$. 
\end{abstract}

\section{Introduction} 

The Monge–Kantorovich (MK) optimal transport (OT) problem  concerns finding an optimal coupling between two distributions $\tg_0,\tg_1$: %
\bbb \label{equ:mk}
\inf_{(X_0,X_1)} %
\E[c(X_1 - X_0)],~~~~s.t.~~~~  \law(X_0)=\tg_0,~~ \law(X_1) = \tg_1, 
\eee 
where 
we seek to 
find (the law of) an optimal coupling $(X_0,X_1)$ of $\tg_0$ and $\tg_1$, for which  marginal laws of $X_0,X_1$ equal  $\tg_0,\tg_1$, respectively, to minimize $\E[c(X_1-X_0)]$, called the $c$-transport cost, 
for 
 a cost function $c$. 
Theories, algorithms, and applications of  optimal transport have attracted a vast literature; see, for example,  the monographs of \cite{villani2021topics,  villani2009optimal, ambrosio2021lectures, santambrogio2015optimal, peyre2019computational} for overviews. %
Notably, OT has been  growing into a popular and powerful technique in machine learning, 
for key tasks such as learning generative models, transfer learning, and approximate inference \citep[e.g.,][]{peyre2019computational, arjovsky2017wasserstein, solomon2014wasserstein, el2012bayesian, courty2014domain, marzouk2016introduction}. 

The OT problem should be treated differently depending on whether $\tg_0,\tg_1$ are discrete or continuous measures. In this work, we focus on the continuous case when $\tg_0,\tg_1$ are high dimensional absolutely continuous measures on $\RR^d$ that are observed through empirical observations, a setting called data-driven OT in \cite{trigila2016data}. 
A well known result in OT \citep[e.g.,][]{villani2009optimal} shows that,  
 if $\tg_0$ is continuous, the optimization in \eqref{equ:mk} can be restricted to the set of deterministic couplings satisfying $X_1 = T(X_0)$ for some continuous transport mapping $T\colon \RR^d\to\RR^d$, 
which is often 
approximated in practice with deep  neural networks %
\citep[e.g.,][]{makkuva2020optimal, korotin2021neural, korotin2022neural, huang2020convex}.

However, continuous OT remains highly challenging computationally. 
One major difficulty  
is to handle the coupling constraints of $\law(X_0) = \tg_0$ and $\law(X_1)=\tg_1$,
which are infinite dimensional when $\tgg$  are continuous. 
As a result, 
\eqref{equ:mk} can not be solved as a ``clean" unconstrained optimization problem. 
There are essentially two types of approaches to solving \eqref{equ:mk} in the literature.  
One uses Lagrange duality to turn \eqref{equ:mk} into a certain minimax game, and the other one approximates 
the constraint with an integral (often entropic-like) penalty function.
However, the minimax approaches suffer from convergence and instability issues and 
are difficult to solve in practice, 
while the regularization approach can not effectively enforce the infinite-dimensional coupling constraints.

\paragraph{This work} 
We present a different approach to 
continuous OT  
that re-frames \eqref{equ:mk} into 
a sequence of simple unconstrained nonlinear least squares optimization problems, 
which 
monotonically reduce the transport cost of a coupling 
while automatically preserving the marginal constraints. 
Different from the 
minimax and regularization approaches that enforce the constraints from outside,  
our method is an \emph{interior} approach 
 which starts from a valid coupling (typically the naive independent coupling), and traverses inside the constraint set to decrease the transport cost.  
Such an interior approach is non-trivial and has not been realized before, because there exists no obvious unconstrained parameterization of the set of couplings of $\tgg$. %

Our method is made possible by leveraging 
\emph{rectified flow} \cite{rectified}, 
a recent approach to constructing (non-optimal) transport maps for generative modeling and domain transfer.
What makes rectified flow special 
is that it provides a simple procedure
that turns a given coupling into a new one that 
obeys the same marginal laws, while yielding no worse transport cost w.r.t. \emph{all} convex functions $c$ simultaneously.  
Despite this attractive property, 
as pointed out in \cite{rectified}, 
rectified flow can not be used to optimize any fixed cost $c$, 
as it is essentially a special \emph{multi-objective} optimization procedure that targets no specific cost.  
Our method is a variant of rectified flow that targets a user-specified cost function $c$ and hence yields a new approach to the OT problem \eqref{equ:mk}. 

\paragraph{Rectified flow} 
We provide a high-level overview of 
the rectified flow  of \cite{rectified} 
 and the main results of this work. 
For a given coupling $(X_0,X_1)$ of $\tgg$,  
the \emph{rectified flow} 
induced by $(X_0,X_1)$ 
is the time-differentiable process $\vv Z = \{Z_t \colon t\in[0,1]\}$ over an artificial notion of time $t\in[0,1]$,
that solves the following ordinary differential equation (ODE): 
\bbb \label{equ:ztflow_0}  
\d \Z_t = v^X_t(\Z_t) \dt,~~~~t\in[0,1],~~~~\text{starting from~~~~} \Z_0 = X_0, 
&& %
\eee  
where $v^X\colon \RR^d \times [0,1]\to \RR^d$  
is a time-dependent velocity field defined as the solution of 
\bbb  
\label{equ:opt0}
\inf_v \int_0^1  \e{ \norm{X_1-X_0 - v(X_t, t)}^2 } \dt, && 
 X_t = t X_1 + (1-t)X_0, 
\eee  
and $X_t$ is the linear interpolation between $X_0$ and $X_1$. 
Eq~\eqref{equ:opt0} is 
a least squares regression problem of predicting the line direction of $(X_1-X_0)$ from every space-time point $(X_t, t)$ on the linear interpolation path, yielding a solution of %
$$
v^X_t(z) = \E\left [X_1-X_0~|~  X_t=z \right ], 
$$
which is 
the average of direction $(X_1-X_0)$ for all lines that pass point $X_t = z$ at time $t$.  
The (conditional) expectations $\E[\cdot]$ above are  w.r.t. the randomness of $(X_0,X_1)$. 
We assume that the solution of \eqref{equ:ztflow_0} exists and is unique, 
and hence 
$v^X_t(z)$ is assumed to exist at least on the trajectories of the ODE. 
The start-end pair $(Z_0,Z_1)$ induced by $\vv Z$ is called the \emph{rectified coupling} of $(X_0,X_1)$,
and we denote it by $(Z_0,Z_1)=\mixupmap((X_0,X_1))$. %

In practice, 
the expectation $\E[\cdot]$ is approximated by empirical observations of $(X_0, X_1)$, and 
 $v$ is approximated by a parametric family, such as deep neural networks. 
In this case, the optimization in Eq~\eqref{equ:opt0} can be solved conveniently with off-the-shelf stochastic optimizers such as stochastic gradient descent (SGD), 
without resorting to minimax algorithms or expensive inner loops. %
This makes rectified flow attractive for deep learning applications as these considered in \cite{rectified}. 

The importance of $(\Z_0,\Z_1) = \map((X_0,X_1))$ is justified 
by two key properties:

1) \emph{$(\Z_0,\Z_1)$ shares the same marginal laws as $(X_0,X_1)$ and is hence a valid coupling of $\tgg$; }

2) \emph{$(\Z_0,\Z_1)$ yields no larger convex transport costs than $(X_0,X_1)$, 
that is, $\E[c(\Z_1-\Z_0)] \leq \E[c(X_1-X_0)]$, 
for \emph{every} convex  function  $c\colon\RR^d\to \RR$.} 

Hence, it is natural to recursively apply the $\map$ mapping, that is, 
       $(Z_0^{k+1}, Z_1^{k+1}) = \map((Z^k_0,Z_1^k))$ starting from $ (Z_0^0, Z_1^0) = (X_0, X_1)$, 
yielding a sequence of 
 couplings
 that is monotonically 
non-increasing in terms of all convex transport costs. %
The initialization can be taken to be 
the independent coupling $(Z_0^0,Z_1^0) \sim \tg_0\times \tg_1$, 
or any other couplings that can be constructed from marginal (unpaired)  observations of $\tgg$. 
In practice, 
each step of $\map$ is empirically approximated by first drawing samples of $(Z_0^k, Z_1^k)$ from the ODE with drift $v^k$, and then constructing the next flow $v^{k+1}$ from the optimization in \eqref{equ:opt0}. Although this process accumulates errors, 
it was shown that one or two iterations are sufficient for practical applications \citep{rectified}.

Note that the $\map$ procedure  is ``cost-agnostic"
 in that it does not dependent on any specific cost $c$.
Although the recursive 
$\map$ update is monotonically non-increasing on the transport cost for all convex $c$, it does not necessarily converge to the optimal coupling for any pre-specified $c$, 
as the update would stop whenever two cost functions are conflicting with each other. 
In \cite{rectified}, 
a coupling $(X_0,X_1)$ is called \emph{straight} if it is a fixed point of $\map$, that is, $(X_0,X_1) = \map((X_0,X_1))$. 
It was shown that rectifiable  couplings that are optimal w.r.t. a convex $c$ must be straight, but the opposite is not true in general. 
One exception is the one dimension case ($d=1$), for which all convex functions $c$ (whose $c$-optimal coupling exists) share a common optimal coupling that is also straight. 
But this does not hold when $d\geq 2$. %

\paragraph{$c$-Rectified flow} 
In this work, 
we modify the $\map$ procedure so that it can be used to 
solve \eqref{equ:mk} given a user-specified cost function $c$. 
We show that this can be done easily by properly restricting the optimization domain of $v$ and modifying the loss function 
in \eqref{equ:opt0}.  
The case of quadratic loss $c(x) = \frac{1}{2}\norm{x}^2$ 
is particularly simple,
for which we simply 
need to restrict the $v$   
to be a gradient field $v_t = \dd f_t$ in the optimization of \eqref{equ:opt0}. 
For more general convex $c$, 
we need to restrict $v$ to have a form of $v_t(x) = \dd c^*(\dd f_t(x))$, with $f$ minimizing the following loss function: 
\bbb\label{equ:minvc}
\inf_{f} 
\int_0^1 \e{ c^*(\dd f( X_t)) - 
(X_1-X_0) \tt \dd f(X_t) + c(X_1-X_0) }\dt,  %
\eee 
where $c^*$ denotes the conjugate function of $c$. 
Obviously when $c(x)=\frac{1}{2}\norm{x}^2$, 
\eqref{equ:minvc} reduces to \eqref{equ:opt0} 
with $v = \dd f$. 
The loss function  in \eqref{equ:minvc} is closely related to \emph{Bregman divergence} \citep[e.g.,][]{banerjee2005clustering} and the so-called \emph{matching loss} \citep[e.g.,][]{auer1995exponentially}. 
We call $\vv Z =\{Z_t\colon t\in[0,1]\}$ that follows $\d Z_t = \dd c^*(\dd f_t(Z_t)) \dt $ with $Z_0 = X_0$ and $f$ solving \eqref{equ:minvc} the $c$-rectified flow of $(X_0,X_1)$, and the corresponding  $(Z_0,Z_1)$ the $c$-rectified coupling of $(X_0,X_1)$, denoted as $(Z_0,Z_1) = \crectify((X_0,X_1))$. 

Similar to the original rectified coupling, 
the $c$-rectified coupling $(Z_0,Z_1)$ also share the same marginal laws as $(X_0,X_1)$ and hence is a coupling of $\tgg$. 
In addition, $(Z_0,Z_1)$ yields no larger transport cost than $(X_0,X_1)$ w.r.t. $c$,  
that is, $\E[c(Z_1-Z_0)] \leq \E[c(X_1-X_0)]$. But this only holds for the specific $c$ that is used to define the flow, rather than all convex functions like $\rectify$. 

More importantly, 
recursively performing $\crectify$ 
allows us to find $c$-optimal couplings that solve the OT problem \eqref{equ:mk}. 
Under mild conditions, we have 
\bb 
(X_0,X_1) = \crectify((X_0,X_1)) 
&&\iff &&
\text{$(X_0,X_1)$ is $c$-optimal in \eqref{equ:mk}}
&&\iff &&
\ell^*_{X,c} = 0,
\ee 
where $\ell^*_{X,c}$ denotes the minimum value of the loss function in \eqref{equ:minvc}, 
which provides a criterion of $c$-optimality of a given coupling without solving the OT problem. Moreover, 
 when following the recursive update  $(Z_0^{k+1},Z_1^{k+1}) = \crectify((Z_0^{k}, Z_1^{k}))$, 
 the $\ell^*_{Z^k,c} $ is guaranteed to decay to zero with  $\min_{k\leq K} \ell^*_{Z^k,c} = \bigO{1/K}$.

\paragraph{Notation} 
Let $\Cone$ be the set of continuously  differentiable functions $f\colon \RR^d\to \RR$, and $\Cc$ the functions in $\Cone$ whose support is compact. 
For a time-dependent velocity field $v \colon \RR^d\times[0,1]\to \RR$, 
we write $v_t(\cdot) = v(x,t)$ and 
use $\dot v_t(x) \defeq \partial v(x,t)$ and $\dd v_t(x) \defeq \partial_x v(x,t)$ to denote the partial derivative w.r.t. time $t$ and variable $x$, respectively. 
We denote by $C^{2,1}(\RR^d\times [0,1])$
the set of functions $f\colon \RR^d\times [0,1]\to \RR$ that are second-order continuously differentiable w.r.t. $x$ and first-order continuously differentiable w.r.t. $t$. 
In this work, 
an ordinary differential equation 
(ODE) $\d z_t = v_t(z_t) \d t$ should be interpolated as an integral equation $z_t = z_0 + \int_0^t v_t(z_t) \dt $. 
For $x\in \RR^d$, $\norm{x}$ denotes the Euclidean norm. We always write $c^*$ as the convex conjugate of $c\colon \RR^d\to \RR$, that is, $c^*(x) = \sup_{y\in\RR^d} \{x\tt y - c(y)\}$. 

Random variables are capitalized (e.g., $X,Y,Z$) to  distinguish them with deterministic values (e.g, $x,y,z$). 
Recall that an $\RR^d$-valued random variable  
$X=X(\omega)$ is a measurable function $X\colon \Omega \to \RR^d$, where $\Omega$ is an underlying sample space equipped with a $\sigma$-algebra $\mathcal F$ and a probability measure $\meas P$. 
The triplet $(\Omega, \mathcal F, \meas P)$ form the underlying probability space, which is omitted in writing in the most places.  
We use $\law(X)$ to denote the probability law of $X$, which is the probability measure $\meas L$ that satisfies $\meas L(B) = \meas P(\{\omega \colon X(\omega) \in B\})$ for all measurable sets on $\RR^d$. 
For a functional $F(X)$ of a random variable $X$, 
the optimization problem $\min_{X}F(X)$ %
technically means to 
find a measurable function $X(\omega)$ to minimize $F$, even though we omit the underlying sample space $\Omega$. When $F(X)$ depends on $X$ only through $\law(X)$, the optimization problem is equivalent to finding the optimal $\law(X)$.

\paragraph{Outline} The rest of the work is organized as follows. Section~\ref{sec:ot} 
introduces the background of optimal transport. Section~\ref{sec:rectopt} 
reviews rectified flow of \cite{rectified} from an optimization-based view. 
Section~\ref{sec:marginal0} characterizes 
the if and only if condition 
for two differentiable stochastic processes to have equal marginal laws. 
Section~\ref{sec:crectify} introduces the main $c$-rectified flow method and establishes its theoretical properties.

\section{Background of Optimal Transport}
\label{sec:ot}
This section introduces the background of optimal transport (OT), 
including both the static and dynamic formulations.  
Of special importance is the dynamic formulation, 
which is closely related to the rectified flow approach. 
The readers can find systematic introductions to OT in 
a collection of excellent textbooks 
\cite{villani2021topics, figalli2021invitation,   ambrosio2021lectures, peyre2019computational, ollivier2014optimal,
santambrogio2015optimal,
villani2009optimal}.  

\paragraph{Static formulations} 
The optimal transport problem was 
first formulated by 
Gaspard Monge in 1781 when he studied 
 the problem of how to 
 redistribute mass, e.g., a pile of soil, with minimal effort. 
 Monge's problem can be formulated as
 \bbb \label{equ:m}
 \inf_{T} \e{c(T(X_0)- X_0)} ~~~~s.t.~~~~ \law(T(X_0))=\tg_1, ~~~ \law(X_0)=\tg_0,
 \eee 
 where we minimize the $c$-transport cost 
 in the set of deterministic couplings $(X_0,X_1)$ that satisfy $X_1 = T(X_0)$ for a transport mapping $T\colon \RR^d\to \RR^d$. 
The Monge–Kantorovich (MK) problem in \eqref{equ:mk} is the relaxation 
of \eqref{equ:m} to the set of all (deterministic and stochastic) couplings of $\tgg$.  
The two problems are equivalent when the optimum of \eqref{equ:mk} is achieved by a 
deterministic coupling, which is guaranteed if $\tg_0$ is  an absolutely continuous measure on $\RR^d$. 

A key feature of the MK problem is that it is
a linear programming w.r.t. the law of the coupling $(X_0,X_1)$, and yields a  dual problem of form: %
\bbb \label{equ:dmk}
\sup_{\mu, \nu} 
\tg_1(\mu)  - \tg_0(\nu) 
~~s.t.~~
 \mu(x_1) - 
\nu(x_0)\leq c(x_1-x_0),~~~~\forall (x_0,x_1),
\eee 
where we write $\tg_1(\mu) \defeq \int \mu(x) \d \tg_1(x)$, and $\mu,\nu$ are optimized in all functions from $\RR^d$ to $\RR$. 
For any coupling $(X_0,X_1)$ of $\tgg$, and $(\mu,\nu)$ satisfying the constraint in \eqref{equ:dmk}, it is easy to see that 
\bbb \label{equ:dualderive}
\E[c(X_1-X_0)] 
\geq \E[\mu(X_1) - \nu(X_0)] 
= \tg_1(\mu) - \tg_0(\nu). 
\eee 
As the left side of \eqref{equ:dualderive} 
only depends on $(X_0,X_1)$ and the right side only on $(\mu,\nu)$, one can show that $(X_0,X_1)$ is $c$-optimal  and $(\mu,\nu)$ solves \eqref{equ:dmk} iff 
$\mu(X_0) + \nu(X_1) = c(X_1-X_0)$ 
holds with probability one, which provides a basic optimality criterion. %
Many existing 
OT algorithms are developed by exploiting the primal dual relation of \eqref{equ:mk} and 
\eqref{equ:dmk} 
(see e.g., \cite{korotin2022neural}), but have the drawback of yielding  minimax problems that are challenging to solve in practice. %
 
If $c$ is strictly convex, 
the %
optimal transport map of \eqref{equ:m} 
is unique (almost surely) and yields a form of 
\bb %
T(x)=x + \dd c^*(\dd \nu(x)), ~~~~
\ee 
where $c^*$ is the convex conjugate function of $c$, and $\nu$ is an optimal solution of \eqref{equ:dmk}, which is  $c$-convex in that $\nu(x) =  \sup_{y}\left \{ -c(y-x) +  \mu(y)\right \}$ with $\mu$ the associated solution. 
In the canonical case of quadratic cost  $c(x) = \frac{1}{2}\norm{x}^2$, we can write $T(x) = \dd \phi(x)$, where $\phi(x) \defeq \frac{1}{2}\norm{x}^2 + \nu(x)$ is a convex function.

\paragraph{Dynamic  formulations} 
Both the MK and Monge problems can be equivalently framed in dynamic ways 
as finding continuous-time processes that transfer $\tg_0$ to $\tg_1$.  
Let $\{x_t \colon t\in[0,1]\}$ be a smooth path connecting $x_0$ and $x_1$, whose time derivative is denoted as $\dot x_t$. 
For convex $c$, by Jensen's inequality, we can represent the cost $c(x_1-x_0)$ in an integral form: 
$$
c({x_1 - x_0}) = c\left (\int_0^1 \dot x_t \dt \right) =  \inf_{x} \int_0^1 c({\dot x_t}) \dt,
$$
where the infimum is attained when $x_t$ is the  linear interpolation (geodesic) path:  $x_t = t x_1 + (1-t)x_0$. 
Hence, the MK optimal transport problem \eqref{equ:mk} is equivalent to 
\bbb \label{equ:qt}
\inf_{\vv X}  
\E\left  [ \int_0^1 c(\dot X_t) \dt \right ]  ~~~~~s.t.~~~~ \law(X_0) = \tg_0, ~~\law(X_1) = \tg_1,
\eee 
where we optimize in the set of  time-differentiable %
stochastic processes 
$\vv X = \{X_t \colon t\in[0,1]\}$.
The optimum of \eqref{equ:qt} 
is attained by  $X_t = t X_1 + (1-t)X_0$ when
 $(X_0,X_1)$ is a $c$-optimal coupling of \eqref{equ:mk}, which is known as the \emph{displacement interpolation}  \citep{mccann1997convexity}. 
 We call the objective function in \eqref{equ:qt} the path-wise $c$-transport cost. 

The Monge problem can also be framed in a dynamic way. 
Assume the transport map $T$ can be induced by an ODE model 
$\d X_t = v_t(X_t)\dt $ such that $X_1 = T(X_0)$. Then the Monge problem is equivalent to %
\bbb \label{equ:cm}
\inf_{v,\traj X} \E\left [ \int _0^1  c({v_t(X_t)}) \dt  \right] ~~~~~s.t.~~~~~ \d X_t = v_t(X_t)\dt ,~~~~~
\law(X_0) = \tg_0, ~~~~~ \law(X_1) = \tg_1, 
\eee 
which is equivalent to restricting $\vv X$  in \eqref{equ:qt} to the set of processes that can be induced by ODEs. 

Assume that $X_t$ following  $\d X_t = v_t(X_t)\dt $ yields a density function $\varrho_t$. Then 
it is well known that $\varrho_t$ satisfies the continuity equation:  %
$$
\dot \varrho_t + \div(v_t\varrho_t) = 0.
$$
Hence, we can rewrite \eqref{equ:cm} into an optimization problem on $(v, \varrho)$, yielding the celebrated \emph{{\bbformula} formula} \cite{benamou2000computational}: 
\bbb \label{equ:bb} 
\inf_{v, \varrho} 
\int_0^1  \int c(v_t(x))  \varrho_t(x)\d x \dt  
~~~~s.t.~~~~ 
\dot  \varrho_t  + \div(v_t \varrho_t) = 0,~~~~\rho_0  = \d\tg_0/\dx, ~~~~\rho_1  = \d \tg_1/\dx, \eee 
where $\d \pi_i/\d x$ denotes the density function of $\tg_i$. 
The key idea of 
\eqref{equ:cm} and \eqref{equ:bb} is to 
restrict the optimization of \eqref{equ:qt} to the set of deterministic processed induced by ODEs, which significantly reduces  the search space. 
Intuitively, 
Jensen's inequality $\E[c(Z)]\geq c(\E[Z])$ shows that we should be able to reduce the expected cost of a stochastic process 
by ``marginalizing'' out the  randomness. 
In fact, 
we will show that, for a differentiable stochastic process $\vv X$, %
its ($c$-)rectified flow yields no larger 
path-wise $c$-transport cost in \eqref{equ:qt} than $\vv X$ (see Lemma~\ref{thm:rectdual} and Theorem~\ref{thm:optmultimarg}).  

However, all the dynamic formulations above are still highly challenging to solve in practice.  
We will show that $c$-rectified flow can be viewed as a special coordinate descent like  approach 
to solving \eqref{equ:qt} (Section~\ref{sec:crectifyOptView}). %

\section{Rectified Flow: An Optimization-Based View}  \label{sec:rectopt}

We introduce rectified flow of \cite{rectified} 
from an optimization-based perspective: 
we  show that rectified flow can be viewed as the solution 
of a special constrained dynamic optimization problem, 
which allows us to gain more understanding of rectified flow and motivates the development of $c$-rectified flow.  

Following \cite{rectified}, 
for a time-differentiable stochastic process $\vv X = \{X_t \colon t\in[0,1]\}$, %
its expected velocity field $v^\X$ is defined as 
\bbb \label{equ:gvxzte}
v^{\vv X}_t(z) = \E[\dot X_t ~|~X_t =z]. 
\eee 
where $\dot X_t$ denotes the time derivative of $X_t$. 
Obviously, $v^{\vv X}$ is the solution of 
\bbb \label{equ:infvLx}
\inf_{v}\left\{  L_{\vv X}(v) \defeq \int_0^1 \e{\norm{\dot X_t - v_t(X_t)}^2}\dt \right \}, 
\eee
where the optimization is on the set of all measurable velocity fields $v\colon \RR^d \to \RR^d$. 
The importance of $v^{\vv X}$ lies on the fact that it characterizes the time-evolution of the marginal laws $\rho_t \defeq \law(X_t)$ of $\vv X$, through the continuity equation in the distributional sense:
\bbb \label{equ:conddeq} 
\partial_t \rho_t  + \div (v^{\vv X}_t \rho_t)=0,~~~~ \rho_0 = \law(X_0),~~~~ t\in[0,1].
\eee 
Precisely, Equation~\eqref{equ:conddeq} should be interpreted by its weak and integral form: 
\bbb \label{equ:weakconddeq}
\rho_t(h) - \rho_0(h) - \int_0^t 
\rho_t(\dd h \tt v^\X_s) \d s =0, ~~~~ \rho_0=\law(X_0),  ~~~~{\forall h \in \Cc},~~~~ t\in[0,1], 
\eee 
where {$\rho_t(h)\defeq \int h(x) \d \rho_t(x)$} and $\Cc$ denotes the set of continuously differentiable functions on $\RR^d$ with compact support.  
Hence, 
if the solution of Eq~\eqref{equ:conddeq}-\eqref{equ:weakconddeq} is unique, then 
the marginal laws $\{\law(X_t)\}_t$ of $\vv X$ are uniquely determined by $v^\X$ and the initial $\law(X_0)$.

We define the rectified flow of $\vv X$, 
denoted by $\vv Z = \rectflow(\vv X)$, 
as the ODE driven by $v^{\vv X}$: 
\bbb 
\label{equ:zofvx}
\d Z_t = v_t^{\vv X}(Z_t)  \dt,~~~~ Z_0 = X_0, ~~~~ t\in[0,1]. 
\eee   
Moreover, the rectified flow of a coupling $(X_0,X_1)$
is defined as the rectified flow of $\vv X$ when $\vv X$ is the linear interpolation of $(X_0,X_1)$. 
\begin{mydef} 
A stochastic process $\vv X$ is called rectifiable if $v^\vv X$ exists and is locally bounded, 
and Equation~\eqref{equ:zofvx} 
has an unique solution.  

A coupling $(X_0,X_1)$ is called rectifiable if its linear interpolation process $\vv X$, following $X_t = t X_1 +(1-t) X_0 $, is rectifiable. 
 In this case, we call $\vv Z= \rectflow(\vv X)$ the rectified flow of $(X_0,X_1)$, and write it (with an abuse of notation) as $\vv Z = \rectflow((X_0,X_1))$. 
 The corresponding 
 $(Z_0,Z_1)$ is called the rectified coupling of $(X_0,X_1)$, denoted as $(Z_0,Z_1) = \map((X_0,X_1))$. 
\end{mydef} 

By the definition in \eqref{equ:zofvx}, 
we have $v^{\vv Z} =v^\X$, 
and hence the marginal laws  $\law(Z_t)$ of $\vv Z$ are governed by the same continuity equation \eqref{equ:conddeq}-\eqref{equ:weakconddeq}, which is a well known fact. As shown in  \citep{kurtz2011equivalence}, Equation~\eqref{equ:zofvx} has an unique solution iff Equation~\eqref{equ:weakconddeq} has an unique solution, 
which implies that $\vv Z$ and $\vv X$ share the same marginal laws. 
We also assumed that the solution of \eqref{equ:infvLx} is unique; 
if not, results in the paper hold for all solutions of \eqref{equ:infvLx}.

\begin{thm}[Theorem~\nn{3.3} of \cite{rectified}]
Assume that $\vv X$ is rectifiable. We have%
\bb 
\vv Z = \rectflow(\vv X) &&\Rightarrow&& 
v^{\vv X} = v^{\vv Z} &&\Rightarrow&&
\law(X_t) = \law(Z_t), ~~\forall t\in[0,1].
\ee 
\end{thm} 
Hence, 
rectified flow
turns a rectifiable %
stochastic process into a flow while preserving  the marginal laws.

\paragraph{A {optimization} view of rectified flow} 
We show that 
the rectified flow $\vv Z$ of $\vv X$ 
achieves the minimum of the path-wise $c$-transport cost %
in the set of time-differentiable stochastic processes whose expected velocity field equals $v^\X$. 
This explains that the property of non-increasing convex transport costs of rectified flow/coupling.
\begin{lem} \label{thm:rectdual}
The rectified flow  $\vv Z = \rectflow(\vv X_t)$ in \eqref{equ:zofvx} attains the minimum of 
 \bbb \label{equ:pathopt}
 \inf_{\vv Y}  
 \left\{ F_c(\vv Y) \defeq \int_0^1  \e{  c(\dot Y_t)    } \dt , ~~~~s.t.~~~~ 
 v^{\vv Y}  = v^\X \right\}, 
 \eee 
 which holds for   \emph{any} convex functions $c\colon \RR^d\to \RR$.  

 \end{lem} 
 \begin{proof} 
 For any stochastic process $\vv Y$ with $v^{\vv X}_t(z) = v^{\vv Y}_t(z) = \E[\dot Y_t | Y_t=z] $, we have 
 \bb  
  F_c(\vv Y) 
  & = \int_0^1\E[c(\dot Y_t)] \dt    \\
  & \geq\int_0^1 \E[c(\E[\dot Y_t|Y_t])] \dt  \ant{Jensen's inequality} \\
  & = \int_0^1 \E[c(v^{\vv Y}(Y_t))]  \dt   \\
  & = \int_0^1 \E[c(v^{\vv X}(X_t))]  \dt  \ant{$v^\X=v^\Y$, and hence $\law(X_t)=\law(Y_t)$}  \\
  & = \int_0^1 \E[c(v^{\vv X}(Z_t))]  \dt  \ant{$\law(X_t) = \law(Z_t)$} \\
 &  = \int_0^1 \E[c(\dot Z_t)\dt  
  = \vv F_c(\vv Z). 
 \ee 
 \end{proof} 
 Lemma~\ref{thm:rectdual} suggests that the rectified flow decreases the path-wise $c$-transport cost: $F_c(\vv Z) \leq F_c(\vv X)$, for all convex $c$. 
Note that $\e{c(Z_1-Z_0)}\leq F_c(\vv Z) $ by Jensen's inequality, 
and $\e{c(X_1-X_0)} = F_c(\vv X)$ 
 if $\vv X$ is the linear interpolation of $(X_0,X_1)$. Hence, in this case, we have %
 $$
\E[c(Z_1-Z_0)] \leq F_c(\vv Z) \leq F_c(\vv X) = \E[c(X_1-X_0)], 
$$
which yields a proof of Theorem~3.2 of \cite{rectified} that the rectified coupling $(Z_0,Z_1)$ yields no larger convex transport costs than $(X_0,X_1)$. %

\paragraph{A primal-dual relation} 
Let us generalize the least squares loss $L_\X(v)$ in \eqref{equ:infvLx} to a  
a Bregman divergence based loss:  
\bb 
\tilde L_{\X,c}(v) \defeq 
\int_0^1 \e{\bcb{\dot X_t; ~ v_t(X_t)}}\dt, && 
\bc(x;y) 
=  c(x) - c(y) - (x-y)\tt \dd c(y),
\ee 
where $\bc(\cdot;\cdot)$ is the Bregman divergence w.r.t. $c$. The least squares loss $L_\X $ is recovered with $c(x) = \frac{1}{2} \norm{x}^2$.

Rectified flow can be alternatively implemented by minimizing $\tilde L_{\X,c}$ with a differentiable strictly convex $c$, 
as in this case the minimum of $\tilde L_{\X,c}$ is also attended by $v^\X(z) = \E[\dot X_t|X_t=z]$.  
The $c$-rectified flow 
is obtained if we minimize $\tilde L_{\X,c}$ with $v$ restricted to be a form of $v = \dd c^*\circ \dd f_t$. 
 See more in Section~\ref{sec:crectify}.

In the following, we show that 
the optimization in 
\eqref{equ:pathopt} can be viewed as the dual problem 
 \eqref{equ:gvxzte}. 
\begin{thm}
For %
any differentiable convex function $c$, 
and {rectifiable process $\vv X$}, we have 
 $$
 \tilde \ell^*_{\vv \X,c} 
\defeq \inf_{v} \tilde  L_{\vv X,c} (v) =   \sup_{\vv Y}\left\{
F_c(\vv X) - F_c(\vv Y)
 ~~ s.t.~~ v^\Y=v^\X \right \}, %
 $$
 and 
 the optima above are achieved when $v = v^{\X}$ and $\vv Y =  \rectflow(\vv X)$.  
\end{thm}
 
\begin{proof} 
Let 
$\var_c(\dot X_t~|~X_t )\defeq \E[c(\dot X_t) - c(\E [\dot X_t|X_t])~|~X_t]
$. 
For any $v$, we have 
\bb 
\tilde L_{\X,c}(v) 
&= \int_0^1 \E[c(\dot X_t) -
c(v(X_t)) - (\dot X_t - v(X_t)) \dd c(v(X_t))] \dt  \\ 
& = 
\int_0^1 \E[c(\dot X_t) -
c(v(X_t)) - 
(v^\X(X_t) - v(X_t)) \dd c(v(X_t))
] \dt \ant{$v^\X(X_t)=\E[\dot X_t |X_t]$}
 \\ 
& \geq 
\int_0^1 \E[c(\dot X_t) -
c(v^\X(X_t)) 
] \dt  \ant{$c(v^\X) \geq c(v) + (v^\X-v) \dd c(v)$} \\ 
& = \int_0^1 \var_c(\dot X_t~|~X_t) \dt , 
\ee 
The inequality is tight when $v = v^\X$, which attains the minimum of $\tilde L_{\X,c}$. 

Write $
R_{\vv X,c}(\vv Y) = F_c(\vv X) - F_c(\vv Y)$. 
We know that $\vv Z = \map(\vv X)$ attains the maximum of $R_{\vv X,c}(\vv Y)$ subject to $v^{\vv Y} = v^{\vv X}$ by Lemma~\ref{thm:rectdual}. In addition, 
\bb 
 R_{\vv X, c}(\vv Z)
 & = \int_0^1 \E[c(\dot X_t) - c(\dot Z_t) ] \dt \\
  & = \int_0^1 \E[c(\dot X_t) - c(v^\X_t(Z_t)])] \dt \\ 
  & = \int_0^1 \E[c(\dot X_t) - c(v^\X_t(X_t)])] \dt \ant{$\law(Z_t)=\law(X_t),\forall t $}\\  
  & = \int_0^1 \E[c(\dot X_t) - c(\E[\dot X_t|X_t])] \dt \\  %
  & =  \int_0^1 \E[\var_c(\dot X_t~|~X_t)] \dt. %
\ee
This concludes the proof. %
\end{proof}

\paragraph{Straight couplings} 
The $\tilde \ell^*_{\vv X,c}=\int_0^1\var_c(\dot X_t|X_t)\dt $ above 
provides a measure of how much the different paths of $\X $ intersect with each other. 
If $c$ is strictly convex and 
$\tilde \ell^*_{\vv X,c} = 0$, we have $\dot X_t = \E[\dot X_t|X_t]$ almost surely, 
meaning that there exist no two paths that go across a point along two different directions. 
In this case, $\vv X$ is a 
fixed point of $\rectflow(\cdot)$, that is, $\vv X = \vv Z = \map(\vv X)$, 
because we have 
$\d X_t =\dot X_t \d t = \E[\dot X_t |X_t]  \dt = v^\X(X_t ) \dt $, which is the same  Equation \eqref{equ:zofvx} that defines $\vv Z$.

Similarly, if $\vv X $ is the linear interpolation of the coupling $(X_0,X_1)$, then 
$\tilde \ell^*_{\X,c} =0$ with strictly convex $c$ if and only if $(X_0,X_1)$ is a fixed point of the $\map$ mapping, that is, $(X_0,X_1) = \map((X_0,X_1))$, 
following 
\cite{rectified}. 
Such couplings are called \emph{straight}, or \emph{fully rectified} in \cite{rectified}.  Obtaining straight couplings 
is useful for learning fast ODE models 
because the trajectories of the associated  rectified flow $\vv Z$ are straight lines and hence can be calculated in closed form without iterative numerical solvers. See \cite{rectified} for more discussion.  %

Moreover, \cite{rectified} showed that 
rectifiable $c$-optimal couplings must be straight. 
In the one dimensional case ($d=1$), 
the straight coupling, if it exists, is unique and attains the minimum of $\E[c(X_1-X_0)]$ for all convex functions for which $c$-optimal coupling exists. For higher dimensions ($d\geq 2$), however, 
straight couplings are not unique, and the specific straight coupling 
obtained at the convergence of the recursive $\map$ update  (i.e. $(Z_0^{k+1},Z_1^{k+1})=\map((Z_0^{k},Z_1^{k}))$) is implicitly determined by the initial coupling $(Z_0^0,Z_1^0)$, 
and  is not expected to be optimal w.r.t. any pre-fixed $c$. 

The following counter example shows a somewhat stronger negative result: 
there exist straight couplings 
that are not optimal w.r.t. all second order differentiable convex functions with invertible Hessian matrices.

\begin{theoremEnd}[proof at the end,
                   no link to proof]{exa}
\label{exa:toycounterexp}                   
Take $\tg_0 = \tg_1 = \normal(0,I)$. 
Hence, for $c(x)=\norm{x}^p$ with $p>0$, 
the $c$-optimal mapping is the trivial identity coupling $(X_0,X_0)$ with $X_0\sim \tg_0$. 

However, 
consider the coupling $(X_0,AX_0)$, 
where $A$ is a 
 non-identity and non-reflecting rotation matrix (namely $A\tt A=I$, $\det(A)=1$, $A \neq I$ and $A$ does not have $\lambda =-1$ as an eigenvalue).  
Then $ (X_0, A X_0)$ is a straight coupling of $\tg_0$ and $\tg_1$, 
but it is not $c$-optimal for all 
second order differentiable convex function $c$ whose Hessian matrix is invertible everywhere. {See Appendix for the proof.}

It is the rotation transform
that makes $(X_0,AX_0)$ sub-optimal, 
which is removed in the proposed $c$-rectified flow in Section~\ref{sec:crectify} via a Helmholtz like decomposition. 
\end{theoremEnd} 
\begin{proofEnd}
i) 
The fact that $A\tt A =I$  and $\tg_0=\tg_1=\normal(0,I)$ 
ensures that $AX_0\sim \tg_1$ and hence 
$(X_0, A X_0)$ is a coupling of $\tg_0$ and $\tg_1$.  
Let $X_t = t A X_0 + (1-t) X_0$ be the linear interpolation of the coupling. Related, we have $\dot X_t = A X_0 - X_0 $. Canceling $X_0$ yields that 
\bbb \label{dotxAIcexample}
\dot X_t =  (A-I) (t A + (1-t)I)^{-1}X_t,
\eee 
where we use the fact that $t A + (1-t)I$ is convertible for $t\in[0,1]$, which we prove as follows: if $t A + (1-t)I$ is not invertible, then $A$ must have $\lambda =-\frac{1-t}{t}$ as one of its eigenvalues;  
but as a rotation matrix, all eigenvalues of $A$ must have a norm of $1$, which means that we must have $t=0.5$ and $\lambda = -1$. This, however, is excluded by the assumption that $A$ is non-reflecting (and hence $\lambda \neq -1$).  

Equation~\eqref{dotxAIcexample} shows that $\dot X_t$ is uniquely determined by $X_t$. Hence, we have $\int_0^1 \e{\var(\dot X_t|X_t) } \dt = 0$, 
which implies that $(X_0,AX_0)$ is a straight coupling by Theorem~\nn{3.6} of \cite{rectified}.  

2) Let $c$ be a second order differentiable convex function whose Hessian matrix $\dd^2 c(x)$ is invertible everywhere. Let $c^*$ be the convex conjugate of $c$, then $c^*$ is also second order differentiable and $\dd c(\dd c^*(x)) = x$, and 
$\dd^2 c^*(x) = \dd^2 c(x)^{-1}.$

If $(X_0,A X_0)$ is a $c$-optimal coupling, there must exists a function $ \phi\colon \RR^d\to \RR$, such that 
\bbb \label{equ:axcstartphix} 
A x = x + \dd c^*(\dd \phi(x)),~~\forall x, 
\eee 
where $c^*$ is the convex conjugate of $c$. 
Equation~\eqref{equ:axcstartphix} is equivalent to $\dd c(Ax - x) = \dd \phi(x)$, which means that $\dd \phi $ is continuously differentiable. 
Taking gradient on both sides of \eqref{equ:axcstartphix} gives
\bbb \label{equ:aihxbx} 
A -I =  H_x B_x, &&& 
H_x = \dd^2 c^*(\dd\phi(x)),~~~~ B_x = \dd^2 \phi(x),
\eee  
where $H_x, B_x$ are both symmetric and $H_x$ is positive definite, 
and hence 
Then $H_x B_x$ is a diagonalizable (all its eigenvalues are real) by Lemma~\ref{thm:diagmat}. 
However, $A-I$ is not diagonalizable because $A$ must have complex eigenvalues as a non-reflecting, non-identity rotation matrix. Hence, \eqref{equ:aihxbx} can not hold. 

\begin{lem}\label{thm:diagmat}
Assume that $A,B$ are two real symmetric matrices and $A$ is positive definite. Then $AB$ is diagonalizable (on the real domain), that is, there exists an invertible matrix $P$, such that $P^{-1} AB P$ is a diagonal matrix.
\end{lem}
\begin{proof}
This is a standard result in linear algebra. 
Because $A$ is positive definite, there exists an invertible symmetric matrix $C$, such that $CC = A.$ Then, $AB = CCB$, and it is similar to $CBC^{-1}$, which is symmetric and hence diagonalizable.
\end{proof}
\end{proofEnd}

\section{Differentiable Processes with Equivalent Marginal Laws}
\label{sec:marginal0}

The marginal preserving property of rectified flow is due to the property of $v^{\vv Z} = v^\X$ by construction. 
However, we show in this section that 
$v^{\vv X} = v^{\vv Z}$ is only a sufficient condition:  
 two differentiable processes $\vv X$ and $\vv Z$ 
can have the same marginal laws 
even if $r\defeq v^\X -v^{\vv\Z} \neq 0$. 
This is because $r $, as illustrated in Example~\ref{exa:toycounterexp},  
can be a rotation-only  vector field  (in a generalized sense shown below) 
that introduces rotation components into the dynamics without modifying the marginal distributions. 
Therefore, the constraint of $v^{\vv Y}=v^\X$ in 
 the optimization problem \eqref{equ:pathopt} may be too restrictive. 
 A natural relaxation of \eqref{equ:pathopt}  would be  
 \bbb \label{equ:pathoptmlaw}
 \inf_{\vv Y}  
 \left\{ F_c(\vv Y) \defeq \e{ \int_0^1  c(\dot Y_t)    \dt} , ~~~~s.t.~~~~ 
 \law(Y_t)  = \law(X_t),~~\forall t\in[0,1]\right\}, 
 \eee 
which yields a dynamic OT problem with %
a continuum of 
marginal constraints. 
 In Section~\ref{sec:crectify}, we show that 
the solution of \eqref{equ:pathoptmlaw} yields our $c$-rectified flow that 
solve the OT problem 
at the fixed point.   
Solving \eqref{equ:pathoptmlaw} allows us to remove the rotational components of $v^\X$, which is what 
what renders rectified flow non-optimal. 
 In this section,  
 we first characterize the necessary and sufficient condition 
 for having equivalent marginal laws. 

\begin{mydef}
A time-dependent vector field $r \colon \RR^d\times [0,1]\to \RR^d$ is said to be $\X$-marginal-preserving if 
\bbb \label{equ:edhxrt0}
\int_0^t\E[\dd h(X_t) \tt r_t(X_t) ] =0,~~~\forall t\in[0,1], ~~~~ h \in \Cc. 
\eee 
\end{mydef} 

Equation~\eqref{equ:edhxrt0} is equivalent to saying 
that $\E[\dd h(X_t) \tt r_t(X_t) ] =0$ holds almost surely assuming that $t$ is a random variable following $\uniform([0,1])$ (i.e., $t$-almost surely). 
Let $\rho_t = \law(X_t)$ and it yields a density function $\varrho_t$. Using integration by parts, we have 
$$
0= 
\E[\dd h(X_t) \tt r_t(X_t) ] = \int \dd h(x) \tt r_t(x)\varrho_t(x) \d x 
= - \int h(x) \div(r_t(x) \varrho_t(x)) \d x, ~~~~\forall h \in \Cc, 
$$
which gives $\div(r_t \varrho_t)=0$. This says that $r_t \varrho_t$ is a  rotation-only (or divergence-free) vector field in the classical sense.

\begin{lem}\label{thm:marginalrot}
 Let $\vv X$ and $\vv Y$  
be two  stochastic processes with
the same initial distributions $\law(X_0) = \law(Y_0)$. 
Assume that $\vv X$ is rectifiable, and $v^{\vv Y}_t(z):=\E[\dot Y_t 
|Y_t = z]$ exists and is locally bounded. %

Then $\X$ and $\Y$ share the same marginal laws at all time, that is, $\law(X_t) = \law(Y_t)$, $\forall t\in[0,1],$ 
if and only if $v^\X - v^\Y$ is $\vv Y$-marginal-preserving. %
\end{lem}
\begin{proof}
Taking any $h$ in $\Cc$,  %
we have for $t\in[0,1]$
\bb
 \E[h(X_t)]  - \E[h(X_0)]
 & = \int_0^t \E[\dd h(X_s)\tt \dot X_s ] \d s  \\ 
 & =\int_0^t \E[\dd h(X_s)\tt v^\X_s(X_s)] \d s   \ant{$v^\X_s(X_s) = \E[\dot X_s |X_s]$} . 
\ee 
This suggests that the marginal law $\rho_t\defeq \law(X_t)$ satisfies 
\bbb \label{equ:tgth0} 
\rho_t(h) - \rho_ 0 (h) - \int_0^t \rho_s (\dd h \tt v^\X_s) \d s =0, ~~~ \forall h \in \Cc ,
\eee 
where we define  
$\rho_t (h) = \int h(x) \d \rho_t(x)$. 
Equation~\eqref{equ:tgth0} is formally written as the continuity equation: 
\bbb  \label{equ:contieq0}
\dot \rho_t + \div (g_t^\X \rho_t) =  0. 
\eee  
Similarly, $\tilde\rho_t\defeq \law(Y_t)$ satisfies 
\bbb \label{equ:tgtph0}
\tilde\rho_t(h) - \tilde\rho_ 0 (h) - \int_0^t \tilde\rho_s (\dd h \tt v^\Y_s) \d s =0, ~~~ \forall h, \eee 
If $v_t^\X - v_t^\Y$ is $\law(Y_t)$-preserving for $\forall t\in[0,1]$, we have 
\bb
\E[h(Y_t)]  - \E[h(Y_0)] 
  & =\int_0^t \E[\dd h(Y_s)\tt \dot Y_s]  \d s  \\ 
 & = \int_0^t \E[\dd h(Y_s)\tt v^\Y_s(Y_s)]  \d s  \\ %
  & = \int_0^t \E[\dd h(Y_s)\tt v^\X_s(Y_s)]  \d s +
  \int_0^t  \E[\dd h(Y_s) \tt (v^\Y_s(Y_s) - v^\X_s(Y_s))]  \d s  \\
 & =\int_0^t  \E[\dd h(Y_s) \tt v^\X_s(Y_s)]  \d s 
 \ant{$v^\X-v^\Y$ is $\vv Y$-preserving},
\ee
which suggests  that $\tilde\rho_t \defeq \law(Y_t)$ 
solves the same continuity equation \eqref{equ:contieq0}, starting from the same initialization as $\law(X_0) = \law(Y_0)$.  
Hence, we have $\rho_t = \tilde\rho_t$  if the solution of \eqref{equ:contieq0} is unique, 
which is equivalent to the uniqueness of the solution of $\d Z_t = \vofX_t(Z_t)$  in \eqref{equ:zofvx} following Corollary 1.3 of \cite{kurtz2011equivalence}. 

On the other hand, if $\rho_t = \law(X_t) = \law (Y_t) = \tilde\rho_t$, 
following \eqref{equ:tgth0}
 and \eqref{equ:tgtph0}, we have for any $h \in \Cc$, 
\bb 
0 
 & = \int_0^t  \tilde\rho_t (\dd h\tt v^{\X}_s) - \tilde\rho_t(\dd h\tt v^{\Y}_s) \d s 
 = \int_0^t \dd h(Y_s) \tt (v^\X(Y_s) -v^\Y(Y_s))\d s, 
\ee 
which is the definition of 
$\vv Y$-marginal-preserving following  \eqref{equ:edhxrt0}.

\end{proof}
\section{$c$-Rectified Flow} %
\label{sec:crectify}
We introduce $c$-rectified flow, 
a $c$-dependent variant of rectified flow that guarantees to 
minimize the $c$-transport cost 
when applied recursively. This section is organized as follows: 
Section~\ref{sec:crectifyDefine} 
defines and discusses the $c$-rectified flow of a differentiable stochastic process $\vv X$, which we show yields the solution of the
 infinite-marginal OT problem~\eqref{equ:pathoptmlaw}. 
 Section~\ref{sec:crectifyCoupling} 
 considers the $c$-rectified flow of a coupling $(X_0,X_1)$, which we show is non-increasing on the $c$-transport cost. 
 Section~\ref{sec:crectifyFixed} 
 proves that the fixed points of $\crectify$  are $c$-optimal. 
 Section~\ref{sec:crectifyOptView}
 interprets $c$-rectified flow 
 as an alternating direction descent method 
 for the dynamic OT problem \eqref{equ:qt}, 
 and a majorize-minimization (MM) algorithm for the static OT problem \eqref{equ:mk}. 
 Section~\ref{sec:hj} discusses 
 a key lemma relating $c$-optimal couplings and its associated displacement interpolation with Hamilton-Jacobi equation.

\subsection{$c$-Rectified Flow of Time-Differentiable Processes $\vv X$}
\label{sec:crectifyDefine}

For a {convex} cost function $c\colon \RR^d\to \RR$ and 
 a time-differentiable process $\vv X$,  
the $c$-rectified flow of $\vv X$, 
 denoted as $\vv Z = \crectflow(\vv X)$, is defined as the solution of 
\bbb \label{equ:zgxft}
\d \Z_t =g^{\X,c}_t(\Z_t) \dt,~~~~ \Z_0 = X_0,~~~~\text{with}~~~~
 g^{\X,c}_t(z) = \dd c^*(\dd f^{\X,c}_t(z)), ~~~~ t\in[0,1],
\eee  
where $c^*(x) \defeq \sup_{y}\{ x\tt y - c(y)\}$ is the convex conjugate of $c$, and 
$f^{\X,c}\colon \RR^d\times [0,1] \to \RR$ is the optimal solution of 
\bbb \label{equ:bregloss} 
\inf_{f} \left\{ 
L_{\X,c} (f) \defeq 
\int_0^1 \E\left [  \mcb{\dot X_t;~ \dd f_t( X_t)}
\right ] \dt \right\}, 
\eee  
where $\mc~ \colon \RR^d\times \RR^d\to [0,+\infty)$ is a loss function defined as %
$$
\mcb{x;y} = c(x) - x\tt y +  c^*(y).
$$
Note that we have $\mcb{x;y}\geq 0$ for $\forall x,y$ following the definition of the conjugate $c^*$ (or the Fenchel-Young inequality). 
Losses of form $\mcb{x;y}$ 
{is equivalent to the so called \emph{matching loss} 
proposed for 
learning generalized linear models \cite{auer1995exponentially}.}    %

Compared with the original rectified flow, 
the difference of $c$-rectified flow is i) restricting the velocity field to a form of $g_t = \dd c^*\circ \dd f_t$, and ii) replacing the quadratic objective function to the matching loss. 
These two changes combined yield a Helmholtz like decomposition of $v^\X$ as we show below, allowing us to remove the ``{rotation-only}" component of $v^\X$ %
and obtain $c$-optimal couplings at fixed points.

\paragraph{Bregman divergence,  Helmholtz decomposition, marginal preserving}
We can equivalently write \eqref{equ:bregloss} using Bergman divergence associated with $c$, that is, %
$$\bcb{x;y}\defeq c(x)-c(y)- \nabla c(y)\tt (x-y).$$ 
Then it is easy to see that $\mcb{x;y} = \bcb{x;  \dd c^*(y)}$, by using the fact that $\dd c(\dd c^*(y)) = y$ and $c^*(y) = y \tt \dd c^*(y) - c(\dd c^*(y))$.
Hence, $\mc$ and $\bc$ are equivalent up to the monotonic transform $\dd c^*$ on $y$. 
The minimum $\bcb{x;y}=0$ is achieved when $y = x$, 
while  $\mcb{x;y}=0$ is achieved when $\dd c^*(y)=x$.

Therefore, \eqref{equ:bregloss} is equivalent to 
\bbb \label{equ:bregloss2} 
\inf_{f} \int_0^1 \E\left [  \bcb{\dot X_t;~ g_t( X_t))}
\right ] \dt , && \text{with~~~~} g_t = \dd c^*\circ \dd f_t. 
\eee 
Moreover,  
the generalized Pythagorean theorem of Bregman divergence (e.g., \citep{banerjee2005clustering}) gives 
\bbb \label{equ:pythabreg}   
\e{\bcb{\dot X_t; ~~ g_t}~|~ X_t} = 
\bcb{\e{\dot X_t|X_t};~~ g_t} +  \e{\bcb{\dot X_t; ~~ \e{\dot X_t|X_t}}}. 
\eee 

Because $v^\X(X_t)= \e{\dot X_t|X_t}$ and the last term of \eqref{equ:pythabreg} is independent with $g_t$ ,
we can further reframe 
\eqref{equ:bregloss} 
into %
\bbb \label{equ:bregvloss}
\min_{f} \int_0^1 \E 
\left [  \bcb{v^\X_t( X_t); ~~ g_t(X_t))}  
\right ] \dt,  && 
 \text{with~~~~} g_t = \dd c^*\circ \dd f_t,  
\eee  
which can be viewed as projecting the expected velocity $v^\X_t$ to the set of functions of form $g_t = \dd c^*\circ \dd f_t$, 
w.r.t. the Bregman divergence. 
This yields an orthogonal decomposition of $v_t^\X$:  
\bbb \label{equ:helm}
v^{\X}_t = \dd c^*\circ \dd f^{\X,c}_t + r^{\X,c}_t,
\eee 
 where $r^{\X,c}_t = v^{\X,c}_t -  \dd c^*\circ \dd f^{\X,c}_t$ is the residual term. 
 The key result below shows that $r^{\X,c}$ is $\X$-marginal-preserving, which ensures that 
 the $c$-rectified flow preserves the marginals of $\X$.

\begin{mydef}
We say that $\vv X$ is $c$-rectifiable if $v^\X$ exists,  
the minimum of \eqref{equ:bregloss} exists and is attained by a locally bounded function $f^{\vv X, c}$,
and  %
the solution of %
 Equation~\eqref{equ:zgxft} exists and is unique. %
\end{mydef}   

\begin{thm}\label{thm:marginalgood}
Assume that $\X$ is $c$-rectifiable, 
and $c^*\defeq \sup_y \{x\tt y - c(y)\}$ and $c^*\in\Cone$. We have 

i) $v^\X- g^{\X, c}$ is $\X$-marginal-preserving. 

ii) $\vv Z = \crectify(\vv X)$ preserves the marginal laws of $\vv X$, that is, 
$\law(Z_t) =\law(X_t)$, $\forall t\in[0,1]$.
\end{thm}
\begin{proof}
i) 
By $v^\X_t(z) = \E[\dot X_t|X_t=z]$, the loss function in \eqref{equ:bregloss} is equivalent to %
\bb 
L_{\X,c}(f)  
& =  \int_0^1 \e{c^*(\dd f_t(X_t))  - \E[\dot X_t~|~X_t ] \tt \dd f_t(X_t ) + c(\dot X_t) } \dt \\ 
& = \int_0^1 \e{c^*(\dd f_t(X_t)) - v^\X_t(X_t) \tt \dd f_t(X_t) + c(\dot X_t)} \dt.  
\ee 
By Euler-Lagrange equation, we have 
$$
\int_0^1 \e{(\dd c^*(\dd f_s^{\X,s}(X_s) ) - v^\X(X_s))\tt \dd g_s(X_s) } \d s = 0, ~~~~ \forall g: ~g_s\in \Cc. 
$$
Taking $g_s = h$ if $s<t$ and $g_s = 0 $ if $s>t$ yields that $r^{\X,c}(x) =\dd c^*(\dd f_s^{\X,c}(X_s) ) - v^\X(X_s)$ is $\vv X$-marginal-preserving following \eqref{equ:edhxrt0}.

ii) Note that $\vv Z$ is rectifiable if $\vv X$ is $c$-rectifiable.  %
Applying Lemma~\ref{thm:marginalrot} 
yields the result. %
\end{proof}

For the quadratic cost $c(x) = c^*(x) = \frac{1}{2} \norm{x}^2$, the $\dd c^*$ is the identity mapping, and \eqref{equ:helm} reduces to the {\helm} decomposition,  
  which represents a velocity field into the sum of a gradient field and a divergence-free field. Hence, 
  \eqref{equ:helm} yields a generalization of {\helm} decomposition, in which a monotonic transform $\dd c^*$ is applied on the gradient field component. 
  We call \eqref{equ:helm} a 
  \emph{Bregman {\helm} decomposition}. 

\paragraph{Remark: score matching}
In some special cases,  $v^\X$ may already be a gradient field, 
and hence the rectified flow and $c$-rectified flow coincide for $c(x) = \frac{1}{2} \norm{x}^2$. One example of this is when $X_t = \alpha_t X_1 + \beta_t \xi$ 
for some time-differentiable functions $\alpha_t$ and $\beta_t$, and $\xi \sim \normal(0,I)$, satisfying $\alpha_1 = 1, \beta_1 = 0$, and $X_0 = \alpha_0 X_1 + \beta_0 \xi$. In this case, one can show that 
\bb 
v_t^\X(z) = 
\E[\dot \alpha_t X_1+\dot \beta_t X_0~|~X_t=z]
= \dd f_t(z),
&&\text{with} &&
f_t(z) = 
\eta_t\log \varrho_t(z) + \frac{\zeta_t}{2} \norm{z}^2,
 \ee 
 where $\varrho_t$ is the density function of $X_t$ with  
 $\varrho_t(z) \propto \int \phi \left (
 \frac{z -\alpha_t x_1}{\beta_t}\right ) \d \tg_1(x_1)$ and $\phi(z) = \exp(-\norm{z}^2/2)$, %
 and 
 $\eta_t = \beta_t^2 (\dot \alpha_t/\alpha_t - \dot \beta_t /\beta_t)$ and $\zeta_t = \dot \alpha_t /\alpha_t$.
 This case covers the probability flow ODEs \citep{song2020score} and denoising diffusion  implicit models (DDIM) \citep{song2020denoising}
 with different choices of $\alpha_t$ and $\beta_t$ %
 as suggested  in \cite{rectified}. 
 When $\zeta_t = 0$,
 as the case of \cite{song2019generative}, 
 $v_t^\X$ is proportional to $\dd\log \rho_t,$ the \emph{score function} of $\varrho_t$,
 and the least squares loss $L_\X(v)$ in \eqref{equ:infvLx} reduces to a time-integrated 
 \emph{score matching} loss \citep{hyvarinen2005estimation, vincent2011connection}. 

 However, 
 $v_t^\X$ is generally not a score function or gradient function, especially in complicate cases  when the coupling $(X_0,X_1)$ is induced from the previous rectified flow  as 
 we iteratively apply the rectification procedure.  
 In these cases, it is necessary to impose the gradient form as we do in $c$-rectified flow. %

\paragraph{$c$-Rectified flow solves Problem~\eqref{equ:pathoptmlaw}}
We are ready to show that the $c$-rectified flow solves the optimization problem in \eqref{equ:pathoptmlaw}. 
Further, \eqref{equ:bregloss} forms a dual problem of 
\eqref{equ:pathoptmlaw}. %

 \begin{thm}\label{thm:optmultimarg}
 Under the  conditions in Theorem~\ref{thm:marginalgood}, we have 
 
 i) $\vv Z = \crectify(\vv X)$ attains the minimum of \eqref{equ:pathoptmlaw}. 
 
 ii) Problem 
 \eqref{equ:pathoptmlaw} and \eqref{equ:bregloss} has a strong duality: 
 $$ 
 \inf_f L_{\X,c}(f) = 
 \sup_{\vv Y}
 \left \{F_c(\vv X) -  F_c(\vv Y)  
 \colon  \law(Y_t) =\law(X_t), ~\forall t\in[0,1]
 \right \}. 
$$
As the optima above are achieved by $f^{\X,c}$ and $\vv Z$, we have 
$L_{\X,c}(f^{\X,c}) = F_c(\vv X) - F_c(\vv Z).$
 \end{thm}
 \begin{proof}
 Write   $R_{\X,c}( \vv Y) = F_c(\vv X) - F_c(\vv Y)$. 
First, we show that $ L_{\X,c}(f) \geq  R_{\X,c}(\Y)$  %
for 
any  $f$ and $\vv Y$ that satisfies $\law(Y_t) = \law(X_t)$, $\forall t$: 
\bb
  & R_{\X,c}(\vv Y) \\ 
 & = \E\left [ \int_0^1  c(\dot X_t) - c(\dot Y_t)   \dt \right]  \\
 & \overset{(1)}{\leq} \E\left [ \int_0^1  c(\dot X_t) + 
 c^*(\dd f_t(Y_t)) 
 - \dot Y_t \tt \dd f_t(Y_t) 
 \dt \right]  \ant{Fenchel-Young inequality: $c(y)\geq x\tt y - c^*(x)$}\\
 & = \E\left [ \int_0^1  c(\dot X_t) + 
 c^*(\dd f_t(Y_t)) 
 - v^\Y_t(Y_t) \tt \dd f_t(Y_t) 
 \dt \right] \!\!\!\!\!\!\! \ant{$v^\Y_t(Y_t)  =\E[\dot Y_t|Y_t]$}\\ 
  & = \E\left [ \int_0^1  c(\dot X_t) + 
 c^*(\dd f_t(X_t)) 
 - v^\Y_t(X_t) \tt \dd f_t(X_t) 
 \dt \right] \!\!\!\!\!\!\!  \ant{$\law(X_t) = \law(Y_t)$} \\
 & = \E\left [ \int_0^1  c(\dot X_t) + 
 c^*(\dd f_t(X_t)) 
 - v^\X_t(X_t) \tt \dd f_t(X_t) 
 \dt \right]  \!\!\!\!\!\!\!\ant{$v^\X-v^\Y$ is $\vv X$-marginal-preserving }\\  %
 & =  L_{\X,c}(f).
\ee 
Moreover, if we take 
$\vv Y = \vv \Z$ and $f = f^{\X,c}$, 
then the inequality in $\overset{(1)}{\leq}$ is tight because $\dot Z_t = \dd c^*(\dd f_t(Y_t))$ holds $t$-almost surely. 
Therefore, $R_{\X,c}(\vv Z) =  L_{\X,c}(f^{\X,c})
\geq R^{\X, c} (Y) $, which suggests that $\vv Z$ attains the maximum of $R_{\X,c}$ (under the marginal constraints) and the strong duality holds. 
 \end{proof}

\subsection{$c$-Rectified Flow of Coupling $(X_0,X_1)$} 
\label{sec:crectifyCoupling}

Similar to the case of rectified flow, 
the $c$-rectified flow/coupling of a coupling $(X_0,X_1)$ 
is defined as the  $c$-rectified flow/coupling of its linear interpolation process.  
In the following, we show that the $c$-rectified coupling of a coupling %
yields no larger
{$c$-transport cost}. %

\begin{mydef}\label{def:lincp}
Let $\vv X$ be the linear interpolation of 
coupling $(X_0,X_1)$ in that $ X_t = t X_1 + (1-t) X_0,\forall t\in[0,1]$.  %
We say that $(X_0,X_1)$ is $c$-rectifiable if $\vv{X}$ is $c$-rectifiable, and call $\vv Z = \crectflow(\vv{ X})$ the $c$-rectified flow of $(X_0,X_1)$. 
We call the induced $(Z_0,Z_1)$ the $c$-rectified coupling of $(X_0,X_1)$, denoted as $(Z_0,Z_1) = \crectify((X_0,X_1))$. 
\end{mydef}

Note that the 
$c$-transport cost $\E[c(X_1-X_0)]$ is related to 
{the path-wise $c$-transport cost $F_c(\vv X)$} via 
\bb 
F_{c}(\vv X) = \E[c(X_1-X_0)] + S_c(\vv X), 
&& S_c(\vv X) \defeq \int_0^1 \E[c(\dot X_t) - c(X_1-X_0)] \dt, 
\ee 
where $S_c(\vv X)$ is a non-negative measurement of how close $\vv X$ is to be {geodesic}: 
We have $S_c(\vv X) \geq 0$
 following Jensen's inequality $\int_0^1 c(\dot X_t)\dt  \geq c(\int_0^1 \dot X_t\dt ) = c(X_1-X_0)$, and $S_c(\vv X) = 0$ if $X_t = t X_1 + (1-t) X_0$. 
 
Hence, when $\vv X$ is the linear interpolation of $(X_0,X_1)$, we have from Theorem~\ref{thm:optmultimarg} that 
 \bbb \label{equ:straighexz}  %
 \E[c(X_1-X_0)] - 
 \E[c(Z_1-Z_0)] = S_c(\vv Z) + L_{\X,c}(f^{\X,c}) \geq 0.
 \eee 
 which establishes that $(Z_0,Z_1)$ yields no larger transport cost than $(X_0,X_1)$.

\begin{thm}\label{thm:cost0}
Assume that $c$ is convex with conjugate $c^* \in \Cone$,    %
 and the conditions in Definition~\ref{def:lincp} holds. Then Equation~\eqref{equ:straighexz} holds 
and $\E[c(Z_1-Z_0)] \leq \E[c(X_1-X_0)].$
\end{thm}
Compared with the regular $\map$ mapping, 
the key difference here is that 
the monotonicity of $\crectify$ only holds for the specific $c$ that it employees, rather than all convex cost functions. More importantly, 
as we show below,  
recursively applying $\crectify$  
yields optimal couplings w.r.t. $c$, 
a key property that the regular rectified flow misses. 

\subsection{Fixed Points of 
$c$-{$\map$} are 
$c$-Optimal}
\label{sec:crectifyFixed}
We show three key results regarding the optimality of fixed points of the $\crectify$ mapping: 

1) A coupling $(X_0,X_1)$ is a 
fixed point of 
$\crectify$, that is, $(X_0,X_1) = \crectify((X_0,X_1))$, if and only if it is  $c$-optimal; 

2) Define $\ell^*_{X,c} = \inf_f L_{\X,c}(f)$ where $\vv X$ is the linear interpolation of $(X_0,X_1)$. 
Then $\ell^*_{X,c}$ yields an indication of $c$-optimality of $(X_0,X_1)$, that is, $L_{X,c}^*=0$, iff $(X_0,X_1)$ is $c$-optimal. 

3) The minimum $\ell^*_{X,c}$ in the first $k$ iterations of  $\crectify$ steps decreases  
with an $O(1/k)$ rate.

\begin{thm}\label{thm:copt}
Assume that $c$ is {convex} with conjugate $c^*$, and $c, c^*\in \Cone$ and   $\vv X$ is the linear interpolation process of $(X_0,X_1)$. 
 Assume that $(X_0,X_1)$ is a $c$-rectifiable coupling, 
 and  $f^{\X,c} \in C^{2,1}(\RR^d\times [0,1])$.   
 Then the following statements are equivalent: 
 
 i) $(X_0,X_1)$ is a fixed point of $\crectify$, that is, $(X_0,X_1) =
 \crectify(X_0,X_1)$. 
 
  ii)  $ \ell^*_{X,c}\defeq \inf_{f} L_{\X,c}(f)  = L_{\X,c}(f^{\X,c}) = 0$, for $L_{\X,c}$ in \eqref{equ:bregloss}. 
 
 iii) $(X_0,X_1)$ is a $c$-optimal coupling.
\end{thm}
\begin{proof}

i) $\to$ ii). 
If $(Z_0,Z_1) = (X_0,X_1)$, we have $S_c(\vv Z)=0$ and $L_{\X,c}(f^{\X,c}) = 0$ following \eqref{equ:straighexz}. 

iii) $\to$ ii). 
If $(X_0,X_1)$ is $c$-optimal, we have $\E[c(X_1-X_0)] \leq \E[c(Z_1-Z_0)]$, which again implies that $L_{\X,c}(f^{\X,c}) =0$ following \eqref{equ:straighexz}.

ii) $\to$ i) 
Note that 
\bb 
L_{\X,c}(f^{\X,c}) = 
\int_0^1\e{\bcb{\dot X_t; ~ g^{\X,c}_t(X_t)}} \dt 
\geq 0.
\ee 
Therefore, $L_{\X,c}(f^{\X,c}) =0$ 
implies that 
$\dot X_t= g^{\X,c}_t(X_t)$ 
$t$-almost surely. Because $Z_t$ satisfies the same equation, whose solution is assumed to be unique, we have $\vv Z =\vv X$ and hence $(Z_0,Z_1) = (X_0,X_1)$. 

ii) $\to$ iii)
Because $\vv X$ is the linear interpolation, we have $X_t = t X_1 + (1-t) X_0$, and {it simultaneously satisfies the ODE $\d X_t= g^{\X,c}_t(X_t) \dt $}.  
Using Lemma~\ref{lem:hjc2} shows that $(X_0,X_1)$ is $c$-optimal. 
\end{proof}

Knowing that $L_{\X,c}(f^{\X,c})$ is an indication of $c$-optimality, we show below that it is guaranteed to converge to zero with recursive $\map$ updates.  
\begin{cor}\label{thm:oneoverk}
Let $\vv Z^{k}$ be the $k$-th $c$-rectified flow of $(X_0,X_1)$, satisfying $\vv Z^{k+1} = \crectflow((Z_0^k,Z_1^k))$  and $(Z_0^0,Z_1^0) = (X_0,X_1)$. 
Assume each $(Z_0^k, Z_1^k)$ is $c$-rectifiable for $k=0,\ldots, K$. Then 
$$
\sum_{k=0}^K  L_{\vv Z^k,c}(f^{\vv Z^k,c}) + 
S_c(\vv Z^{k+1}) \leq \E[c(X_1-X_0)]. 
$$
Therefore, if $\E[c(X_1-X_0)]<+\infty$, 
we have $\min_{k\leq K} L_{\vv\Z^k,c}(f^{\vv \Z^k,c}) + 
S_c(\vv Z^{k+1}) = \bigO{1/K}$. 
\end{cor}
\begin{proof}
Applying \eqref{equ:straighexz} to $(Z_0^k,Z_1^k)$  and $(Z_0^{k+1},Z_1^{k+1})$ yields 
$$  L_{\vv Z^k,c}(f^{\vv Z^k,c}) + 
S_c(\vv Z^{k+1}) = \E[c(Z_1^k-Z_0^k)] - \E[c(Z_1^{k+1}-Z_0^{k+1})]. 
$$
Summing it over $k=0,\ldots, K$, 
\bb 
\sum_{k=0}^K L_{\vv Z^k,c}(f^{\vv Z^k,c}) + 
S_c(\vv Z^{k+1})
& = \sum_{k=0}^K \E[c(Z_1^k-Z_0^k)] - \E[c(Z_1^{k+1}-Z_0^{k+1})]  \\
& = \E[c(Z_1^0-Z_0^0)] - \E[c(Z_1^{K+1} - Z_0^{K+1})] \\
& \leq \E[c(X_1-X_0)]. 
\ee 
\end{proof}

\subsection{
$c$-Rectified Flow as Optimization Algorithms} 
\label{sec:crectifyOptView} 
In this section, we draw more understanding on how iterative $c$-rectified flowing solves the static and dynamic OT problems. %
We first show that $c$-rectified flow can be viewed as an alternative direction descent on the dynamic OT problem 
\eqref{equ:qt}, and then that $c$-rectified coupling as a majorize-minimization (MM) algorithm on the statistic OT problem~\eqref{equ:mk}.  
The results in this section are framed in terms of a general path-wise loss function $F_c(\vv Y)$, 
and hence provide a useful starting point for deriving $c$-rectified flow like approaches to   
more general optimization problems with coupling constraints.

\paragraph{$c$-Rectified flow as alternative direction descent on \eqref{equ:qt}} 
The mapping $\vv Z^{k+1} = \crectflow(\vv Z^k)$  can be 
interpreted as an alternative  direction descent procedure for the dynamic OT problem \eqref{equ:qt}: %
\bbb 
 &  \vv X^k = \argmin_{\vv Y} \left \{  F_c(\vv Y) ~~~s.t.~~~ (Y_0,Y_1) =(Z_0^k,Z_1^k)  \right\}, \label{equ:linF0} \\
 & \vv \Z^{k+1} = \argmin_{\vv Y} \left \{ F_c(\vv Y) ~~~s.t.~~~ \law(Y_t) = \law(X_t^k), ~~ \forall t\in[0,1]  \right\}.  \label{equ:mpF0}
\eee  
Here in \eqref{equ:linF0}, 
we  minimize $F_c(\vv Y)$ 
in the set of processes whose start-end pair $(Y_0,Y_1)$ equals the coupling $(Z_0^k, Z_1^k)$ from $\vv Z^k$, 
which simply yields the linear interpolation $X_t^k = t Z^k_1 + (1-t)Z^k_0$ by Jensen's inequality. 
In \eqref{equ:mpF0}, we minimize $F_c(\vv Y)$ given the path-wise marginal constraint of $\law(Y_t)=\law(X_t^k)$ for all time $t\in[0,1]$, which yields the $c$-rectified flow following Theorem~\ref{thm:optmultimarg}. 
Note that the updates in both \eqref{equ:linF0} and \eqref{equ:mpF0}  
keep the start-end marginal laws $\law(Y_0)$ and $\law(Y_1)$ unchanged, and hence 
the algorithm stays inside the feasible set $\{\vv Y \colon \law(Y_0)=\tg_0, \law(Y_1)=\tg_1\}$ 
in \eqref{equ:qt} 
once it is initialized to be so. 

The updates in \eqref{equ:linF0}-\eqref{equ:mpF0} highlight a key difference between our method and the {\bbformula} approach~\eqref{equ:cm}-\eqref{equ:bb}: 
the key idea of {\bbformula} is to restrict the optimization domain to the set of deterministic, ODE-induced processes (a.k.a. flows),  
but our updates alternate between the 
deterministic 
$c$-rectified flow  $\vv Z^k$ and the linear interpolation process $\vv X^k$, which is \emph{not} deterministic or ODE-inducable 
unless the fixed point is achieved.

\paragraph{$c$-Rectified flow as an MM algorithm} 
The majorize-minimization (MM) algorithm  \citep{hunter2004tutorial} 
is a general optimization recipe that 
works by finding a surrogate function that \emph{majorizes} the objective function. 
Let $F(X)$ be the objective concave function to be minimize. An  MM algorithm consists of iterative update of form $X^{k+1} \in \argmin_Y F^+(Y~|~X^k)$, 
where $F^+$ is a majorization function of $F$ that satsifies
$$
F(Y) = \min_{X} F^+(Y~|~X), ~~~~ \text{and the minimum is attained when $X = Y$}. 
$$
In this case, 
the MM update  guarantees that %
$F(X^k)$ is monotonically non-increasing: 
$$
F(X^{k+1}) \leq F^+(X^{k+1} | X^k) \leq F^+(X^k~|~X^k) = F(X^k).  
$$
One can also view MM as conducting coordinate descent on $(X,Y)$ for solving $\min_{X,Y} F^+(Y~|~X)$.

In the following, we show that $(Z_0^{k+1},Z_1^{k+1}) = \crectify((Z_0^k,Z_1^k))$ 
can be interpreted as an MM algorithm for the static OT problem \eqref{equ:mk}  for minimizing $\E[c(X_1-X_1)]$ in the set of couplings of $\tgg$. %
The majorization function corresponding to $\crectify$ can be shown to be 
\bb 
F^+_c((Y_0,Y_1)~|~(X_0,X_1))
& = \inf_{\tilde{\vv Y}} 
\left \{ F_c(\tilde{\vv Y})  ~~~s.t.~~~ 
 (\tilde Y_0,\tilde Y_1) = (Y_0, Y_1), ~~~  \vv Y\in \mathcal M_X
\right \},  \\ 
& \text{with}~~~~\mathcal M_X = \{\vv Y \colon ~~
\law(Y_t) = \law(t X_1 + (1-t)X_0),~~\forall t\in[0,1]\}, 
\ee 
where $F^+_c((Y_0,Y_1)~|~(X_0,X_1))$ denotes the minimum value of  $F_c(\tilde{ \vv Y}) $ 
for $\tilde{\vv Y}$ whose start-end points equal $(Y_0, Y_1)$, 
and yields the same marginal laws as that of the linear interpolation process of $(X_0,X_1)$.

\begin{pro} 
i) 
$F^+_c$ 
 yields a majorization function of the $c$-transport cost $\E[c(Y_1-Y_0)]$ in the sense that 
 $$\displaystyle 
 \E[c(Y_1-Y_0)]=\min_{(X_0,X_1)} \{ F^+_c((Y_0,Y_1)~|~(X_0,X_1)), {~~s.t.~~ (X_0,X_1) \in \Pi_{0,1}}\}, 
 $$ 
 and the minimum is attained by $(X_0,X_1) = (Y_0,Y_1)$,
 {where $\Pi_{0,1}$ denotes the set of couplings of $\tgg$.}
 
 ii) %
 $\crectify$ yields the MM update related $F^+$ in that 
 $$
 \displaystyle \crectify((X_0,X_1)) \in \argmin_{(Y_0,Y_1)\in \Pi_{0,1}}F^+_c((Y_0,Y_1)~|~(X_0,X_1)). 
 $$
\end{pro} 
\begin{proof}
i) For any coupling $(X_0,X_1)$ and $(Y_0,Y_1)$, we have
$$
F^+_c((Y_0,Y_1)|(X_0,X_1)) \geq
\inf_{\tilde{\vv Y}} \left \{ \vv F_c(\tilde{\vv Y}) ~~~s.t.~~~ (\tilde Y_0, \tilde Y_1) = (Y_0,Y_1) \right\} 
= \E[c(Y_1-Y_0)], 
$$
where the inequality holds because remove the constraint $\vv Y\in \mathcal M_X$. %
In addition, it is obvious that the inequality above becomes equality when $(X_0,X_1) = (Y_0,Y_1)$. 

ii) Note that 
$$
\inf_{(Y_0,Y_1)}F^+_c((Y_0,Y_1)~|~(X_0,X_1)) = \inf_{\vv{ Y}} 
\left \{ F_c(\vv{ Y})  ~~~s.t.~~~ 
 \vv Y\in \mathcal M_X
\right \}, 
$$
whose minimum of the right side is attained by $\vv Y = \crectflow((X_0,X_1))$  following  Theorem~\ref{thm:optmultimarg}. Hence, the minimum of the left side is attained by $(Y_0,Y_1) = \crectify((X_0,X_1))$.  
\end{proof}

\subsection{Hamilton-Jacobi Equation and Optimal Transport}
\label{sec:hj}
The proof of Theorem~\ref{thm:copt} relies on a 
key lemma shows that if the trajectories of an ODE of form $\d X_t = \dd c^*(\dd f_t(X_t))\dt $ are geodesic in that $X_t = tX_1+(1-t) X_0$, 
then the induced coupling $(X_0,X_1)$ is an $c$-optimal coupling of its marginals. 
The proof of this lemma relies on 
Hamilton-Jacobi (HJ) equation, which provides 
a characterization of $f$ 
for an ODE $\d X_t = \dd c^*(\dd f_t(X_t))\dt $ whose trajectories are geodesic. 
The connection between HJ equation and optimal transport has been a classic result and 
can be found in, for example, \cite{villani2021topics,villani2009optimal}.

\begin{lem}\label{lem:hjc2}
Let $v_t(x) = \dd c^*(\dd f_t(x))$ where $c^* \in C^1(\RR^d)$
is a convex function $c$, %
and $f \in C^{2,1}(\RR^d\times [0,1])$ and $\dd c^*$ is an injective mapping. Assume all trajectories of $\d x_t = v_t(x_t)\dt$ are geodesic paths in that $x_{t}=t x_1 +(1-t)x_0$. Then we have:

i) There exists $\tilde f_t$ such that $\dd \tilde f_t=\dd f_t$ (and hence we can replace $f$ with $\tilde f$ in the assumption), 
such that the following Hamilton–Jacobi (HJ) equation holds 
\bbb \label{equ:hj}
\partial_t \tilde f_t(x)  + c^*(\dd \tilde f_t(x)) =0,~~~\forall x\in \RR^d, ~~ t\in[0,1], && 
\text{(HJ equation)}. 
\eee 

ii)  %
$f$ satisfies 
\bb 
f_t(y_t) = \inf_{y_0 \in \RR^d} \left\{  t c\left (\frac{y_t-y_0}{t} \right ) + f_0(y_0) \right \} 
, ~~\forall t \in [0,1], ~~~ y_t\in\RR^d, 
&& \text{(Hopf-Lax formula)}
\ee 
where the minimum is attained if $\{y_t\}$ follows the ODE $\d y_t = v_t(y_t)\dt $. 

iii) 
Assume a coupling $(X_0,X_1)$ 
of $\tg_0,\tg_1$ 
satisfies $\d X_t = v_t(X_t)\dt$. 
Then $(X_0, X_1)$  is a $c$-optimal coupling. %
\end{lem}
\begin{proof}
i) 
Starting from any point $x_t = x \in\RR^d$ at time $t$, 
because the trajectories of $\d x_t = v_t(x_t) \dt $ are geodesic, 
we have $\dot x_t =  v_t(x_t) 
=\const$ following the trajectory. Because $v_t(x) = \dd c^*(\dd f_t(x))$ and 
$\dd c^*$ is injective, we have $\dd f_t(x_t) = const$ as well. Hence, we have 
\bb 
0 = \ddt \dd f_t(x_t) 
& = \partial_t \dd f_t(x_t) + \dd ^2f_t(x_t)  \dot x_t \\ 
& = \partial_t\dd f_t(x_t) + \dd ^2f_t(x_t)  \dd c^*(\dd f_t(x_t)). 
\ee 
On the other hand, define  $h_t(x) = \partial_t f_t(x) + c^*(\dd f_t(x))$. Then we have 
\bb 
\dd_x h_t(x) & =  
\partial_t \dd f_t(x) + \dd^2 f_t(x_t) \dd c^*(\dd f_t(x)) = 0. 
\ee 
This suggests that $\dd_x h_t(x) =0$ everywhere and hence 
$h_t(x)$ does not depend on $x$.   
Define $\tilde f_t(x) =  f_t(x) - \int_0^t h_t(x_0) \d t $, where $x_0$ is any fixed point in $\RR^d$. 
Then 
$$
\tilde h_t(x) \defeq 
\partial_t \tilde f_t(x) + c^*(\dd \tilde f_t(x)) 
= h_t(x) - h_t(x_0) = 0.$$

ii) 
Take any $y_0,y_1$  in $\RR^d$, 
let $y_t = t y_1 + (1-t) y_0$ be their linear interpolation. We have 
\bb 
& f_1(y_1) - f_0(y_0)  \\ 
& = \int_0^1 (\partial_t f_t(y_t) +  \dd f_t(y_t) \tt (y_t - y_0)) \dt  \\
& = \int_0^1 \dd f_t(y_t)\tt (y_1 - y_0) - c^*(\dd f_t(y_t))  \dt \ant{$h_t = \partial f_t + c^*(\dd f_t) = 0$}\\
&\overset{(1)}{\leq} \int_0^1 c(y_1- y_0)  \dt 
\ant{$c(x) + c^*(y) \geq x\tt y$}\\
& =  c(y_1- y_0). %
 \ee 
  The equality in 
   $\overset{(1)}{\leq}$ is attained if $y_t$ follows the geodesic ODE $\d y_t  = v_t(y_t) \dt $ as we have  $y_1 -y_0 = \dd c^*(\dd f_t(y_t))$, $\forall t$ in this case. 
  A similar derivation holds for $f_t$.

iii) Note that i) gives that $c(y_1-y_0) \geq f_1(y_1) - f_0(y_0) $. 
For any coupling $(Y_0,Y_1)$ of $\tg_0,\tg_1$, we have 
\bb 
\E[c(Y_1-Y_0)] \geq \E[f_1(Y_1) - f_0(Y_0)] 
= \E[f_1(X_1) - f_0(X_0)] 
= \E[c(X_1- X_0)]. 
\ee 
Hence, $(X_0,X_1)$ is a $c$-optimal coupling. 
\end{proof}

\paragraph{Connection to {\bbformula} Formula} 
The results in Lemma~\ref{lem:hjc2} 
can also formally derived from {\bbformula} problem \eqref{equ:bb}, 
as shown in the seminal work of \cite{benamou2000computational}.
By introducing a Lagrangian multiplier $\lambda \colon \RR^d\times [0,1]\to \RR$ for the constraint of $\dot \varrho_t + \div (v_t \varrho_t) =0$,  the problem in \eqref{equ:bb}  can be framed into a minimax problem: 
$$
\inf_{v,\varrho} 
\sup_{\lambda }
\left\{ \mathcal L(v,\varrho, \lambda)  \defeq 
\int c(v_t)
\varrho_t + \int  \lambda_t  (\dot \varrho_t + \div (v_t \varrho_t) ~~~~s.t.~~~~ \varrho \in \Gamma_{0,1}
\right\}, 
$$
where $\mathcal L(v,\varrho, \lambda) $ is the Lagrangian function, and 
$\Gamma_{0,1}$ denotes the set of density functions $\{\varrho_t\}_t$ satisfying $\varrho_0=\d \tg_0/\dx, ~ \varrho_1 = \d \tg_1/\dx$. 
Note that the following integration by parts formulas: 
\bb 
\int \lambda_t \div (v_t\varrho_t)+ \dd \lambda_t \tt v_t\varrho_t = 0, && \int \lambda_t \dot \varrho_t   + \dot \lambda_t \varrho_t= 
\lambda_1\varrho_1 - \lambda_0 \varrho_0, %
\ee 
where we assume that $\lambda_t v_r \rho_t$ decays to zero sufficiently fast at infinity. 
We have 
$$
\mathcal L(v,\varrho, \lambda) =
(\lambda_1 \varrho_1 -
\lambda_0  \varrho_0) +
\int (c\circ v_t)  
\rho_t - \dot \lambda_t  \varrho_t -  \dd \lambda_t \tt (v_t \varrho_t).
$$
At the saddle points, 
the functional derivations of $\mathcal L$ equal zero, yielding 
\bb
\frac{\delta \mathcal L}{\delta \varrho_t}
 = c(v_t) - \dot \lambda_t - \dd \lambda_t \tt v_t = 0,   && 
\frac{\delta \mathcal L}{\delta v_t} 
= (\dd c(v_t) -  \dd \lambda_t)\varrho_t = 0. 
\ee 
Assume $\varrho_t$ is positive everywhere and note that $\dd c^*(\dd c(x)) = x$, we have $v_t = \dd c^*(\dd \lambda_t)$, 
and hence $\dd \lambda_t\tt v_t - c( v_t) = c^*(\dd \lambda_t)$. 
Plugging it back to $\frac{\delta \mathcal L}{\delta \rho_t}
=0$ yields that $\dot \lambda_t + c^*(\dd \lambda_t)=0$. 
Overall, the (formal) KKT condition of \eqref{equ:bb} is
\bb 
& \dot \varrho_t + \div (v_t \varrho_t) = 0, ~~~\rho_0 = \d \tg_0/\dx, ~~\rho_1 = \d\tg_1/\dx  \ant{coupling condition} \\ 
& v_t = \dd  c^*
(\dd \lambda_t) \ant{mapping is gradient of convex function} \\
& \dot \lambda_t + c^*\left ({\dd \lambda_t}\right ) = 0. \ant{Hamilton-Jacobi equation} 
\ee 
This matches the result in Lemma~\ref{lem:hjc2} with $\lambda_t = \tilde f_t$.

\section{Discussion and Open Questions}

\begin{enumerate}
    \item Corollary~\ref{thm:oneoverk} only bounds the surrogate measure $\ell^*_{Z^k, c}$. Can we directly bound the optimality gap on the $c$-transport  cost 
    $e^*_k = \E[c(Z_1^k-Z_0^k)] - \inf_{(Z_0,Z_1)} \E[c(Z_1-Z_0)]$?
    Can we find a certain type of strong convexity like condition, under which $e_k^*$ decays exponentially with $k$? 

    \item For machine learning (ML) tasks such as generative models and domain transfer, the transport cost is not necessarily the direct object of interest. 
    In these cases, 
    as suggested in 
    \cite{rectified}, rectified flow might be preferred because it is simpler and does not require to specify a particular cost $c$.  
    Question: for such ML tasks, when would it be preferred to use OT with a specific $c$, and how to choose $c$ optimally? 

    \item In practice, recursively applying the ($c$-)rectification accumulates errors
    because the training optimization for the drift field and the simulation of the ODE 
  can not be conducted perfectly. 
  How 
  to avoid the error accumulation
  at each step? 
  Assume $\{x_{1,i}\}_i\sim \tg_1$, and $\{z_{0,i}^k, z_{1,i}^k\}_i$ is obtained by solving the ODE of the $k$-th $c$-rectified flow starting from $z_{0,i}^k \sim \tg_0$. 
  As we increase $k$,  $\{z_{0,i}^k\}_i$ may yield increasingly bad approximation of $\tg_1$ due to the error accumulation. One way to fix this is to adjust $\{z_{1,i}^k\}$ to make it closer to $\{x_{1,i}^k\}_i$ at each step. This can be done by reweighting/transporting $\{z_{1,i}^k\}_i$  towards $\{x_{1,i}^k\}_i$ by minimizing certain discrepancy measure, 
  or replacing each $z_{1,i}^k$ with $x_{\sigma(i)}^k$ where $\sigma$ is a permutation that yields a   one-to-one matching between $\{z_1\datai\}$ and $\{x_1\datai\}_i$. 
  The key and challenging part is to do the adjustment in a good and fast way, ideally with a (near) linear time complexity.  
  
  \item With or without the adjustment step, 
  build a complete theoretical analysis on the statistical error of the method. 
  
  \item In what precise sense is rectified flow solving a multi-objective variant of optimal transport? 
    
\end{enumerate}

\bibliography{z_diffusion_models} 

\newcommand{\etalchar}[1]{$^{#1}$}
\begin{thebibliography}{SSDK{\etalchar{+}}20}

\bibitem[ABS21]{ambrosio2021lectures}
Luigi Ambrosio, Elia Bru{\'e}, and Daniele Semola.
\newblock {\em Lectures on optimal transport}.
\newblock Springer, 2021.

\bibitem[ACB17]{arjovsky2017wasserstein}
Martin Arjovsky, Soumith Chintala, and L{\'e}on Bottou.
\newblock Wasserstein generative adversarial networks.
\newblock In {\em International conference on machine learning}, pages
  214--223. PMLR, 2017.

\bibitem[AHW95]{auer1995exponentially}
Peter Auer, Mark Herbster, and Manfred~KK Warmuth.
\newblock Exponentially many local minima for single neurons.
\newblock {\em Advances in neural information processing systems}, 8, 1995.

\bibitem[BB00]{benamou2000computational}
Jean-David Benamou and Yann Brenier.
\newblock A computational fluid mechanics solution to the monge-kantorovich
  mass transfer problem.
\newblock {\em Numerische Mathematik}, 84(3):375--393, 2000.

\bibitem[BMD{\etalchar{+}}05]{banerjee2005clustering}
Arindam Banerjee, Srujana Merugu, Inderjit~S Dhillon, Joydeep Ghosh, and John
  Lafferty.
\newblock Clustering with bregman divergences.
\newblock {\em Journal of machine learning research}, 6(10), 2005.

\bibitem[CFT14]{courty2014domain}
Nicolas Courty, R{\'e}mi Flamary, and Devis Tuia.
\newblock Domain adaptation with regularized optimal transport.
\newblock In {\em Joint European Conference on Machine Learning and Knowledge
  Discovery in Databases}, pages 274--289. Springer, 2014.

\bibitem[EMM12]{el2012bayesian}
Tarek~A El~Moselhy and Youssef~M Marzouk.
\newblock Bayesian inference with optimal maps.
\newblock {\em Journal of Computational Physics}, 231(23):7815--7850, 2012.

\bibitem[FG21]{figalli2021invitation}
Alessio Figalli and Federico Glaudo.
\newblock {\em An Invitation to Optimal Transport, Wasserstein Distances, and
  Gradient Flows}.
\newblock 2021.

\bibitem[HCTC20]{huang2020convex}
Chin-Wei Huang, Ricky~TQ Chen, Christos Tsirigotis, and Aaron Courville.
\newblock Convex potential flows: Universal probability distributions with
  optimal transport and convex optimization.
\newblock {\em arXiv preprint arXiv:2012.05942}, 2020.

\bibitem[HD05]{hyvarinen2005estimation}
Aapo Hyv{\"a}rinen and Peter Dayan.
\newblock Estimation of non-normalized statistical models by score matching.
\newblock {\em Journal of Machine Learning Research}, 6(4), 2005.

\bibitem[HL04]{hunter2004tutorial}
David~R Hunter and Kenneth Lange.
\newblock A tutorial on mm algorithms.
\newblock {\em The American Statistician}, 58(1):30--37, 2004.

\bibitem[KLG{\etalchar{+}}21]{korotin2021neural}
Alexander Korotin, Lingxiao Li, Aude Genevay, Justin~M Solomon, Alexander
  Filippov, and Evgeny Burnaev.
\newblock Do neural optimal transport solvers work? a continuous wasserstein-2
  benchmark.
\newblock {\em Advances in Neural Information Processing Systems},
  34:14593--14605, 2021.

\bibitem[KSB22]{korotin2022neural}
Alexander Korotin, Daniil Selikhanovych, and Evgeny Burnaev.
\newblock Neural optimal transport.
\newblock {\em arXiv preprint arXiv:2201.12220}, 2022.

\bibitem[Kur11]{kurtz2011equivalence}
Thomas~G Kurtz.
\newblock Equivalence of stochastic equations and martingale problems.
\newblock In {\em Stochastic analysis 2010}, pages 113--130. Springer, 2011.

\bibitem[LGL22]{rectified}
Xingchao Liu, Chengyue Gong, and Qiang Liu.
\newblock Flow straight and fast: Learning to generate and transfer data with
  rectified flow.
\newblock {\em preprint}, 2022.

\bibitem[McC97]{mccann1997convexity}
Robert~J McCann.
\newblock A convexity principle for interacting gases.
\newblock {\em Advances in mathematics}, 128(1):153--179, 1997.

\bibitem[MMPS16]{marzouk2016introduction}
Youssef Marzouk, Tarek Moselhy, Matthew Parno, and Alessio Spantini.
\newblock An introduction to sampling via measure transport.
\newblock {\em arXiv e-prints}, pages arXiv--1602, 2016.

\bibitem[MTOL20]{makkuva2020optimal}
Ashok Makkuva, Amirhossein Taghvaei, Sewoong Oh, and Jason Lee.
\newblock Optimal transport mapping via input convex neural networks.
\newblock In {\em International Conference on Machine Learning}, pages
  6672--6681. PMLR, 2020.

\bibitem[OPV14]{ollivier2014optimal}
Yann Ollivier, Herv{\'e} Pajot, and Cedric Villani.
\newblock {\em Optimal Transportation: Theory and Applications}.
\newblock Number 413. Cambridge University Press, 2014.

\bibitem[PC{\etalchar{+}}19]{peyre2019computational}
Gabriel Peyr{\'e}, Marco Cuturi, et~al.
\newblock Computational optimal transport: With applications to data science.
\newblock {\em Foundations and Trends{\textregistered} in Machine Learning},
  11(5-6):355--607, 2019.

\bibitem[San15]{santambrogio2015optimal}
Filippo Santambrogio.
\newblock Optimal transport for applied mathematicians.
\newblock {\em Birk{\"a}user, NY}, 55(58-63):94, 2015.

\bibitem[SE19]{song2019generative}
Yang Song and Stefano Ermon.
\newblock Generative modeling by estimating gradients of the data distribution.
\newblock {\em Advances in Neural Information Processing Systems}, 32, 2019.

\bibitem[SME20]{song2020denoising}
Jiaming Song, Chenlin Meng, and Stefano Ermon.
\newblock Denoising diffusion implicit models.
\newblock In {\em International Conference on Learning Representations}, 2020.

\bibitem[SRGB14]{solomon2014wasserstein}
Justin Solomon, Raif Rustamov, Leonidas Guibas, and Adrian Butscher.
\newblock Wasserstein propagation for semi-supervised learning.
\newblock In {\em International Conference on Machine Learning}, pages
  306--314. PMLR, 2014.

\bibitem[SSDK{\etalchar{+}}20]{song2020score}
Yang Song, Jascha Sohl-Dickstein, Diederik~P Kingma, Abhishek Kumar, Stefano
  Ermon, and Ben Poole.
\newblock Score-based generative modeling through stochastic differential
  equations.
\newblock In {\em International Conference on Learning Representations}, 2020.

\bibitem[TT16]{trigila2016data}
Giulio Trigila and Esteban~G Tabak.
\newblock Data-driven optimal transport.
\newblock {\em Communications on Pure and Applied Mathematics}, 69(4):613--648,
  2016.

\bibitem[Vil09]{villani2009optimal}
C{\'e}dric Villani.
\newblock {\em Optimal transport: old and new}, volume 338.
\newblock Springer, 2009.

\bibitem[Vil21]{villani2021topics}
C{\'e}dric Villani.
\newblock {\em Topics in optimal transportation}, volume~58.
\newblock American Mathematical Soc., 2021.

\bibitem[Vin11]{vincent2011connection}
Pascal Vincent.
\newblock A connection between score matching and denoising autoencoders.
\newblock {\em Neural computation}, 23(7):1661--1674, 2011.

\end{thebibliography}
\bibliographystyle{alpha} 

\appendix 
\section{Proofs} 
\printProofs

\end{document}